\documentclass{article}

\usepackage{microtype}
\usepackage{graphicx}
\usepackage{booktabs}
\usepackage{hyperref}

\usepackage[accepted]{icml2026}

\usepackage{amsmath}
\usepackage{amssymb}
\usepackage{amsthm}

\theoremstyle{plain}
\newtheorem{proposition}{Proposition}
\theoremstyle{definition}

\newcommand{\vecop}{\operatorname{vec}}
\newcommand{\R}{\mathbb{R}}
\newcommand{\Z}{Z}

\icmltitlerunning{The Gradient Does Not See Rank}

\begin{document}

\twocolumn[
\icmltitle{The Gradient Does Not See Rank:\\
Rank-Indifference in Matrix-CODI on ProsQA}

\begin{icmlauthorlist}
\icmlauthor{Samuel Larson}{pebble}
\end{icmlauthorlist}

\icmlaffiliation{pebble}{Pebble ML}

\icmlcorrespondingauthor{Samuel Larson}{\mbox{samlarson@pebbleml.com}}

\icmlkeywords{continuous chain-of-thought, matrix-valued latents,
mechanistic interpretability, rank, superposition}

\vskip 0.3in
]

\printAffiliationsAndNotice{}  

\begin{abstract}
Continuous chain-of-thought models compress reasoning into latent
tokens. Matrix-valued variants, which route each latent token through a
$d\times d$ matrix bottleneck, introduce rank as a single-sample
structural observable on the latent matrix $\Z$. If matrix latents
carry parallel reasoning paths via superposition, rank should track
them, and truncating $\Z$ to low rank should hurt accuracy on tasks
whose solutions plausibly require multiple components. Across four
training regimes of a matrix-CODI model (three on ProsQA, one on
GSM8K-Aug below the learning threshold), the rank-$k$ projection
ablation curve is flat to within 0.6 percentage points. A three-seed
replication yields $81.0 \pm 2.0$pp accuracy while the final
effective rank of $\Z$ spans $\{4, 12, 13\}$; the loss does not
reward any particular rank. To test whether rank-blindness arises
from the flatten-then-project readout alone, we trained four
readouts: a bilinear reparametrization, a bilinear-plus-GELU readout
nonlinear in $\Z$, an SVD-augmented readout feeding singular values
through an MLP, and a quadratic readout in $\Z\Z^\top$. All four
rank-$k$ curves remain flat (Spearman $p$-values $0.63, 0.14, 0.82,
0.46$). The flat curves persist for readouts nonlinear in $\Z$. A
linear probe on $\Z$ underperforms a raw pretrained hidden state
at target prediction (AUC $0.673$ vs.\ $0.846$). A negative control
on vanilla GPT-2 SFT (no matrix bottleneck, no $\Z$, three seeds,
$n\!=\!500$) reproduces a flat rank-$k$ curve under the same
intervention paradigm with pooled-mean range $0.20$pp, and a
random-$h$ sensitivity floor lands at the same accuracy: the
rank-$k$ ablation alone conflates rank-blindness with
position-irrelevance.
\end{abstract}

\section{Introduction}
\label{sec:intro}

Continuous chain-of-thought (CoT) models replace explicit textual
reasoning steps with continuous latent tokens fed back into the
transformer's residual stream. COCONUT \citep{hao2024coconut} and
CODI \citep{shen2025codi} are representative instances: both
compress an explicit rationale into a small number of continuous
latent positions and decode the answer from the resulting state.
Theoretical work \citep{zhu2025superposition,gozeten2025cot2} argues
that these latents can hold multiple reasoning paths in
\emph{superposition}, so continuous CoT could explore a search tree
in parallel.

\citet{rizvi2026illusion} challenged this: a fine-tuned COCONUT model
reaches 96.6\% on ProsQA \emph{without} feeding back any latent
tokens, against 99.0\% with latents and 85.3\% for explicit CoT. They
named this the \emph{Illusion of Superposition}.

In this paper we study \emph{matrix-CODI}, our variant of CODI in
which each continuous latent token passes through a $d\times d$
matrix bottleneck on the feedback path: the 768-dim latent state is
projected to $d^2$ numbers, reshaped into a matrix $\Z$, and read
back out to 768 dims before being fed back as the next latent
position (\S\ref{sec:background}). The latent token is otherwise
CODI's; what the matrix adds is a structural observable, since a
$d\times d$ thought $\Z$ has a computable rank via its SVD. If each singular direction encodes a separate
reasoning path, truncating $\Z$ to rank $k$ at inference should
degrade accuracy when the task needs more than $k$ paths, making
the rank-$k$ ablation curve a natural probe.

But rank counts parallel paths only under conditions the matrix
parametrization does not guarantee: the stored features must align
with the singular directions of $\Z=\mathrm{reshape}(W_{\text{up}}h)$,
and the training objective must reward that alignment. Superposition,
moreover, need not be orthogonal: features packed at small mutual
angles need not raise the rank, so the number of superposed features
need not equal the rank. We therefore test directly whether
matrix-CODI training makes rank a functional readout of reasoning. It
does not.

\subsection{Contributions}

We report five results on a matrix-CODI bottleneck (GPT-2 small,
$d{=}16$, six latent positions, ProsQA):

\begin{itemize}
\setlength\itemsep{2pt}
\item Rank-$k$ ablation is flat at two distillation weights, with
the multiplicative thinker on or off, and on GSM8K-Aug as well as
ProsQA. Range across $k\!\in\!\{1,2,4,8,16\}$ is $\leq 0.6$pp.
\item Three seeds at otherwise identical hyperparameters land at
effective ranks $\{4, 12, 13\}$ and accuracy of
$80.99\!\pm\!2.0$pp (seeds 1337/42/7: $78.91/81.25/82.81\%$).
\item Four nonlinear-in-$\Z$ readouts (bilinear, bilinear+GELU,
SVD-augmented, quadratic in $\Z\Z^\top$) also give flat curves;
Spearman $p$ in $[0.14,\,0.82]$.
\item A linear probe on $\Z$ (1536 features across six positions)
reaches AUC $0.673$ on ProsQA target prediction; a pretrained
GPT-2 hidden state at 768 features reaches $0.846$ on the same
task.
\item Rank-$k$ on vanilla GPT-2 SFT (no $\Z$, three seeds)
reproduces a flat curve under the same intervention paradigm;
pooled range $0.20$pp.
\end{itemize}

The seed-level rank spread is what distinguishes the rank-blindness
reading from position-irrelevance.

\section{Background}
\label{sec:background}

\subsection{The matrix-CODI bottleneck}

\paragraph{CODI distillation.} CODI \citep{shen2025codi} trains a
student to compress an explicit chain-of-thought into a fixed number
of continuous latent positions by matching a hidden-state target
from a teacher pass. The teacher pass consumes prompt $+$ CoT $+$
answer and produces a reference hidden state at a designated
colon-token position. The student pass consumes prompt $+\;n$ latent
positions $+$ answer, where each latent position is produced by
feeding the previous step's hidden state back as the next input
embedding. A hidden-state L1 loss (the \emph{distillation} loss)
aligns the student's state at the answer colon to the teacher's,
and a standard next-token cross-entropy loss trains the answer
prediction. The total loss is $\mathcal{L}=\gamma\mathcal{L}_{\text{kd}}+\mathcal{L}_{\text{ce}}$.

\paragraph{Why matrix latents.} A vector of dimension $D$ has no
single-sample notion of how many independent components it
superposes; a $d\times d$ matrix has rank as a structural observable.
Matrix-valued memory has a long lineage \emph{inside} attention and
SSM layers (\S\ref{sec:related}). Matrix-CODI is different in
placement: the matrix sits on the explicit chain-of-thought
feedback path, one $d\times d$ thought per latent reasoning
position. We test whether the rank of that matrix behaves as a
per-step count of reasoning paths under CODI-style training.

\paragraph{Matrix bottleneck.} We extend CODI with a
\emph{matrix bottleneck} on the latent feedback path. Given the
previous latent hidden state $h\in\R^{D}$ (here $D=768$ for GPT-2
small):
\begin{align*}
\text{flat} &= W_{\text{up}}\,h,\qquad W_{\text{up}}\in\R^{d^2\times D}\\
\Z &= \text{reshape}(\text{flat};\,d,d)\in\R^{d\times d}\\
\Z &\leftarrow (I+\Delta(\Z))\,\Z\,(I+\Gamma(\Z)) \quad\text{(optional thinker)}\\
h_{\text{out}} &= \text{LayerNorm}(\phi(\Z)),
\end{align*}
where $\phi:\R^{d\times d}\to\R^{D}$ is the \emph{readout}. The
default readout is $\phi(\Z)=W_{\text{down}}\vecop(\Z)$, which we
call the \emph{flatten-then-project} readout. $W_{\text{down}}\in
\R^{D\times d^2}$. We use $d=16$ throughout.

\subsection{Probing rank}

\paragraph{Rank-$k$ ablation.} At inference, compute the SVD
$\Z = U\Sigma V^\top$ and replace $\Z$ by its rank-$k$ truncation
\[
\Z_k \;=\; U_{:,\,:k}\,\Sigma_{:k,\,:k}\,V_{:,\,:k}^\top
\]
before the readout. If rank is functional, accuracy should drop
as $k$ decreases to 1. If rank is vestigial, the curve is flat.

\paragraph{Effective rank.} The numerical rank of $\Z$ is always
$d$ since it is a dense trained matrix. For training curves we
report \emph{effective rank}, a smooth spectral proxy:
\begin{align*}
H_{\text{eff}}(\Z) &\;=\; \exp\!\left(-\sum_i \tilde\sigma_i \log \tilde\sigma_i\right),\\
\tilde\sigma_i &\;=\; \sigma_i \big/ \textstyle\sum_j \sigma_j.
\end{align*}
The ablation itself uses hard top-$k$ truncation. The question of
interest is whether the structural capacity offered by rank $>1$
is functionally used.

\paragraph{ProsQA.} ProsQA \citep{hao2024coconut} is a synthetic
entailment task: a diamond-shaped directed acyclic graph of
entities and a question \emph{which leaf entity has property $P$?}.
Each problem has a unique positive answer and a single distractor.
We use the training split from the original COCONUT release. Every
"best" accuracy below is the maximum over 25 per-epoch evaluations on
the first 128 problems of the 500-problem COCONUT test file (resolution
$0.78$pp per problem); the linear probe and negative control draw on
the full 500-problem held-out set, noted where used.

\section{The Flatten-Then-Project Readout is Rank-Blind}
\label{sec:flat}

\subsection{Four flat rank-$k$ curves}

We ran the matrix bottleneck under four training conditions that
vary the task, the CODI distillation weight $\gamma$, and the
multiplicative thinker. Each run produced a single rank-$k$ ablation
curve.

\begin{table}[t]
\caption{Rank-$k$ projection ablation across four training
conditions. $\Z$ rank is the mean effective rank of the $16\times16$
matrix thought at eval time. $r_s$ is the Spearman rank correlation
between per-sample effective rank and correctness. Range across
$k\in\{1,2,4,8,16\}$ is $\leq 0.6$pp in every row. R1 (GSM8K-Aug)
is at a $6\%$ operating point below the learning threshold and is
not interpretable on its own.}
\label{tab:four-flat}
\centering
\footnotesize
\setlength{\tabcolsep}{2pt}
\begin{tabular}{llcclccc}
\toprule
Run & Task & $\gamma$ & Thinker & $\Z$ rank & k{=}1 & k{=}16 & $r_s$ \\
\midrule
R1 & GSM8K-Aug & 1.0 & on  & 5.5  & 6.00\% & 6.12\% & $-0.023$ \\
R2 & ProsQA    & 1.0 & on  & 10.2 & 78.4\% & 78.4\% & $+0.026$ \\
R3a & ProsQA   & 0.0 & on  & 12.7 & 76.8\% & 76.6\% & $-0.105$ \\
R3b & ProsQA   & 0.0 & off & 12.8 & 72.6\% & 72.4\% & $+0.095$ \\
\bottomrule
\end{tabular}
\end{table}

Varying the task (arithmetic vs.\ logical), the distillation weight
(removing the L1-at-colon loss raises effective rank from
$\sim\!10$ to $\sim\!13$ but leaves the curve shape unchanged), or
the multiplicative thinker (off drops accuracy by $\sim\!1$pp; curve
still flat) does not bend any of the four curves. A vanilla GPT-2 SFT
model with no $\Z$ at all also produces a flat curve under the same
probe (\S\ref{sec:pc-negctrl}). The three-seed decoupling result
(Fig.~\ref{fig:seed-decoupling}) below is the model-level evidence
the negative control cannot reproduce by construction.

\subsection{Why: the readout Jacobian is constant in $\Z$}

The flatten-then-project readout is
$\phi(\Z) = W_{\text{down}}\vecop(\Z)$. Its Jacobian with respect
to $\vecop(\Z)$ is the constant matrix $W_{\text{down}}$:
\[
\frac{\partial\phi}{\partial\vecop(\Z)} \;=\; W_{\text{down}}.
\]
By the chain rule, the loss gradient is
\[
\frac{\partial\mathcal{L}}{\partial\vecop(\Z)}
\;=\;
W_{\text{down}}^{\top}\,
\frac{\partial\mathcal{L}}{\partial\phi}.
\]
The pullback factor $W_{\text{down}}^{\top}$ is independent of
$\Z$. Any $\Z$-dependence of the gradient enters only through the
upstream factor $\partial\mathcal{L}/\partial\phi$, which depends
on $\Z$ only through the value $\phi(\Z)$, not through any
explicit factor of $\Z$'s SVD basis.

\begin{proposition}[Linear readout has constant Jacobian]
\label{prop:jac}
Let $\phi:\R^{d\times d}\to\R^{D}$ be linear in $\Z$. The Jacobian
$\partial\phi/\partial\vecop(\Z)$ is a constant matrix
$W\in\R^{D\times d^2}$, and the loss gradient is
$\partial\mathcal{L}/\partial\vecop(\Z) = W^{\top}\,\partial\mathcal{L}/\partial\phi$.
The pullback factor $W^{\top}$ is independent of $\Z$. Any
rank-dependence of $\partial\mathcal{L}/\partial\Z$ enters only
through the upstream factor $\partial\mathcal{L}/\partial\phi$,
which sees $\Z$ only through $\phi(\Z)$. The chain rule therefore
introduces no term that couples directly to the SVD basis of
$\Z$; the loss has no built-in rank reward. Implicit bias from the
optimizer (Adam $+$ weight decay) and upstream regularization may
still shape rank through channels outside $\mathcal{L}$
(\S\ref{sec:related}).
\end{proposition}

\begin{proof}[Proof sketch]
Differentiating $\phi(\Z) = W\,\vecop(\Z)$ gives
$\partial\phi/\partial\vecop(\Z) = W$, a constant matrix. The
chain rule then yields
$\partial\mathcal{L}/\partial\vecop(\Z) = W^{\top}\,\partial\mathcal{L}/\partial\phi$;
the prefactor $W^{\top}$ does not depend on $\Z$. \qedhere
\end{proof}

\paragraph{Remark.} The proposition is a statement about the
chain-rule pullback, not about the full $\Z$-dependence of the
gradient. Different $\Z$'s produce different
$\partial\mathcal{L}/\partial\phi$ and therefore different
gradients. What is ruled out is an SVD-basis-dependent term in
the chain rule itself. Whether the absence of an explicit rank
preference yields rank-blind behavior in practice is an empirical
question; \S\ref{sec:pc} tests it by replacing $\phi$ with
nonlinear functions, and the flat curves persist, suggesting
at least part of the explanation lies elsewhere.

The four training conditions in Table~\ref{tab:four-flat} are
consistent with the proposition's prediction for linear $\phi$.
\S\ref{sec:pc} tests what happens when the linearity premise is
violated.

\subsection{Accuracy is decoupled from rank across seeds}

Proposition~\ref{prop:jac} predicts that the loss landscape should
be flat along rank-changing directions of $\Z$: a model ending
training at rank $12$ and a model ending at rank $4$ can achieve the
same accuracy. We test this with three training seeds of the same
flatten-readout configuration (gpt2-small, ProsQA, $\gamma=0$, 25
epochs, batch 16).

\begin{figure}[t]
\centering
\includegraphics[width=\linewidth]{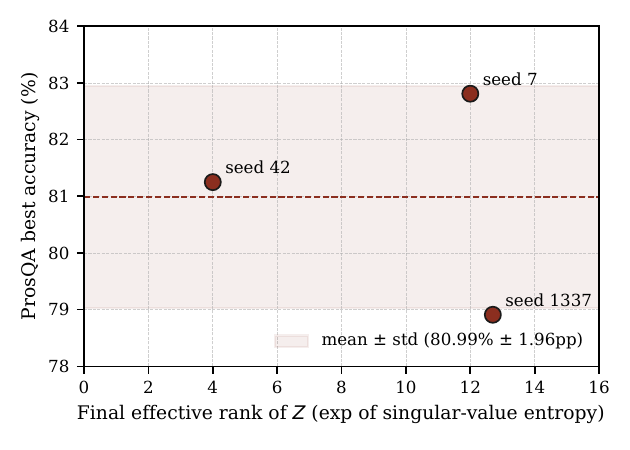}
\caption{Three-seed replication of the flatten readout. Best ProsQA
accuracy is $80.99\pm 2.0$pp (seeds 1337/42/7: $78.91/81.25/82.81\%$), but the final effective rank
of $\Z$ varies $3\times$: seed 42 converges at rank $\sim\!4$;
seeds 1337 and 7 converge near rank $12$. The loss does not push
$\Z$ toward any particular rank.}
\label{fig:seed-decoupling}
\end{figure}

\subsection{Linear probe on $\Z$}

If rank is vestigial, what \emph{does} $\Z$ carry? A 5-fold
cross-validated multi-class logistic regression predicts the ProsQA
target class from a flattened $\Z$ on $500$ held-out test problems.
Controls: the same prompt's hidden state from a pretrained GPT-2
with no ProsQA fine-tuning, and the same model at random
initialization.

\begin{figure}[t]
\centering
\includegraphics[width=\linewidth]{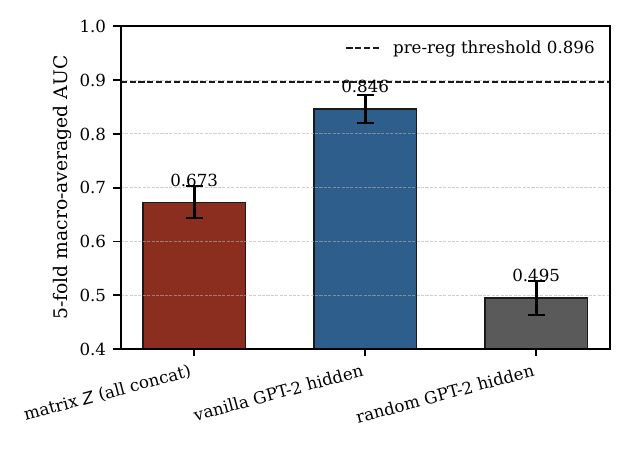}
\caption{Linear probe AUC for ProsQA target class prediction.
Pre-registered threshold for a positive result was
$\max(\text{vanilla},\text{random})+0.05=0.896$. The matrix $\Z$
concat AUC of $0.673$ does not exceed it. Vanilla GPT-2, never
trained on ProsQA, predicts the target class better than the
trained matrix-CODI bottleneck.}
\label{fig:probe}
\end{figure}

Vanilla pretrained GPT-2 reaches AUC $0.846$ at 768 features. The
matrix-CODI bottleneck's concatenated matrix thought (6 latent
positions $\times$ 256 features $= 1536$ features) reaches AUC
$0.673$: more features than the vanilla hidden state, lower
predictive signal for the ProsQA target. A dimension-matched
comparison (probe on the post-bottleneck reconstructed 768-dim
hidden state $W_{\text{down}}\vecop(\Z)$ that the downstream
transformer consumes) is pending. A binary target-vs-distractor
probe on the same $\Z$ tensors is at chance (AUC $0.50$--$0.56$)
across all conditions.

\section{Depth and Scale Do Not Rescue Matrix-CODI}
\label{sec:depth-scale}

The results in \S\ref{sec:flat} and \S\ref{sec:pc} may be specific to
$d=16$, GPT-2 small, and six latent positions. We test two axes: depth
(number of iterative latent refinement steps) and backbone scale.

\subsection{Depth sweep}

\paragraph{Depth sweep (preliminary).} At $n\!=\!6$ latent
refinement steps, vanilla CODI reaches $78.91\%$ on ProsQA,
$\sim\!2.9$pp below pure SFT. The $n\!\in\!\{16,32,64\}$
configurations exceeded memory at the default batch sizes; re-runs at
smaller batches are in progress. A single non-baseline data point is
not a depth sweep, so we report $n\!=\!6$ as a single number and
do not draw a trend.

\subsection{Scale sweep}

We trained vanilla SFT and matrix-CODI ($d=16$, six latent positions,
$\gamma=0$) on ProsQA at three backbone sizes: GPT-2 small (124M),
GPT-2 medium (355M), and GPT-2 large (774M). Matrix-CODI at GPT-2
large exceeded memory at batches of 2 and 4; it is omitted from
Fig.~\ref{fig:scale}.

\begin{figure}[t]
\centering
\includegraphics[width=\linewidth]{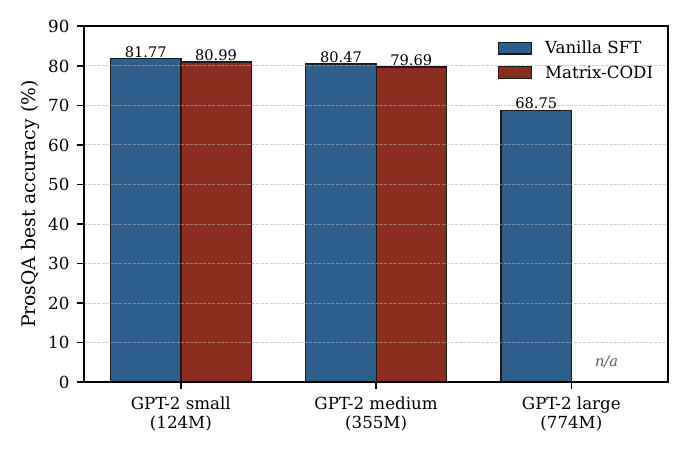}
\caption{Scale sweep on ProsQA. Vanilla SFT degrades at GPT-2 large
($68.75\%$) compared to GPT-2 small ($81.77\%$, three-seed mean).
ProsQA ($17{,}886$ training examples) likely under-optimizes the
larger backbone at default AdamW $\text{lr}=10^{-4}$. Matrix-CODI's
best accuracy is below its matched vanilla SFT baseline at both tested
scales; gaps are within three-seed standard deviation. GPT-2 large
matrix-CODI is pending.}
\label{fig:scale}
\end{figure}

Three observations from Fig.~\ref{fig:scale}:

\begin{itemize}
\setlength\itemsep{2pt}
\item \textbf{Matrix-CODI does not exceed vanilla SFT at any tested
      scale.} Gap is $-0.78$pp at GPT-2 small (three-seed means, $80.99$ vs $81.77$; the single seed-1337 cell reads $-2.86$pp) and $-0.78$pp at GPT-2
      medium, both within the three-seed standard deviation of
      $\pm 2.0$pp measured at GPT-2 small (\S\ref{sec:flat}). The
      sign is consistent across two scales. Matrix-CODI at GPT-2
      large is pending.
\item \textbf{Vanilla SFT itself degrades with scale on ProsQA.}
      The ProsQA training set ($17{,}886$ examples) is small relative
      to GPT-2 large's capacity; default AdamW at $\text{lr}=10^{-4}$
      likely under-optimizes the larger backbone. This is a data-size
      artifact of ProsQA, not a finding about superposition.
\item \textbf{Matrix does not rescue the regression.} If matrix-CODI's
      inductive bias were relevant at scale, we would expect it to
      preserve or improve on its gpt2-small performance as capacity
      grows. It does not.
\end{itemize}

This sweep establishes that matrix-CODI does not exceed a matched
vanilla SFT baseline at the tested scales and training configurations.
It does not rule out settings in which matrix bottlenecks help at
larger scale under different training regimes.

\paragraph{Sample efficiency.} A sample-efficiency sweep (200, 500,
2000, 5000, $17{,}886$ training examples) puts matrix-CODI strictly
below vanilla SFT at every $N$ below $17{,}886$, with the gap growing
as $N$ shrinks (matrix at $N=200$: $12.6\%$; vanilla at $N=200$:
$26.0\%$). Full sample-efficiency table is in the release.

\section{Positive Control: Nonlinear Readouts Also Produce Flat Curves}
\label{sec:pc}

Proposition~\ref{prop:jac} gives a sufficient condition for the
readout Jacobian to be constant in $\Z$. We test the prediction by
training four readouts that violate it and measuring the rank-$k$
curve on each.

\subsection{Four readout variants}

We replace only $\phi:\R^{d\times d}\to\R^D$ in the matrix bottleneck
and leave all other training hyperparameters fixed (gpt2-small,
ProsQA, $\gamma=0$, six latent positions, $d=16$, 25 epochs, batch
16, seed $1337$, AdamW at $\text{lr}=10^{-4}$). All four are trained
at $\gamma=0$ (no CODI L1-at-colon term), so the flat curves below
are about the cross-entropy loss through the matrix bottleneck, not
specifically the CODI distillation loss.

\paragraph{Bilinear.} $\phi(\Z) = W\,\operatorname{probes}(\Z)$
where $\operatorname{probes}(\Z)_k = u_k^\top \Z v_k$ for $K=d^2$
learned probe pairs $(u_k,v_k)\in\R^d\times\R^d$. Each probe is a
Frobenius inner product $\langle u_k v_k^\top,\Z\rangle_F$, linear
in $\Z$. A reparametrization control: it checks that the low-rank
factoring of $W_{\text{down}}$ does not change the curve.

\paragraph{Bilinear+GELU.} $\phi(\Z) = W\,\operatorname{GELU}(
\operatorname{probes}(\Z))$. The GELU gates each probe by a scalar
that depends on $\Z$, making $\phi$ nonlinear in $\Z$. The Jacobian
is not constant, but its column space in $\R^{d\times d}$ is fixed
(the span of the $u_k v_k^\top$); only column magnitudes vary with
$\Z$. A mild violation of Proposition~\ref{prop:jac}; SVD-augmented
(next) is a stronger one.

\paragraph{SVD-augmented.} $\phi(\Z) = W_{\text{down}}\vecop(\Z) +
\operatorname{MLP}(\sigma(\Z))$, where $\sigma(\Z)\in\R^d$ is the
vector of singular values fed through two dense layers with GELU.
We compute $\sigma(\Z)$ via \texttt{torch.linalg.svdvals}, whose
backward is documented as unconditionally numerically stable (in
contrast to full \texttt{torch.linalg.svd}, whose backward is
unstable at near-coincident singular values). This variant
explicitly exposes rank to the optimizer. SVD-augmented's best
accuracy ($78.12\%$) is below the flatten baseline ($78.91\%$, a one-problem margin on the 128-problem split), so
the optimizer is not zeroing the \texttt{sigma\_proj} branch and
falling back to flatten alone.

\paragraph{Quadratic.} $\phi(\Z) = W_{\text{down}}\vecop
\operatorname{concat}(\Z\Z^\top,\Z^\top\Z)$. Quadratic in $\Z$ and
linear in the second-moment tensor. The gradient with respect to
$\Z$ depends on $\Z$ itself.

\subsection{All four rank-$k$ curves are flat}

We computed the rank-$k$ ablation on each trained checkpoint at
$k\in\{1,2,4,8,16\}$ on $128$ ProsQA test problems (the standard
CODI eval split).

\begin{figure}[t]
\centering
\includegraphics[width=\linewidth]{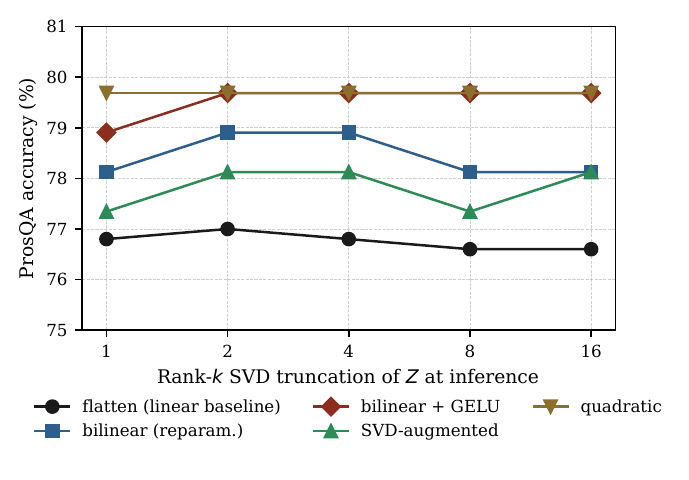}
\caption{Rank-$k$ projection ablation for five readouts on ProsQA.
The flatten (linear) baseline is the Round 3 $\gamma=0$ run. All
four positive-control readouts, including the explicitly
nonlinear-in-$\Z$ Bilinear+GELU, the SVD-augmented readout that
feeds singular values through an MLP, and the quadratic readout in
$\Z\Z^\top$, produce curves flat to within $\sim\!0.8$pp. The
Quadratic readout is perfectly flat at $79.69\%$ across all five
$k$.}
\label{fig:pc}
\end{figure}

\begin{table}[t]
\caption{Rank-$k$ ablation accuracies (\%) by readout on $128$
ProsQA test problems. Spearman $r_s$ of per-sample effective rank
against correctness; $p$ is two-sided. None significant.}
\label{tab:pc}
\centering
\footnotesize
\setlength{\tabcolsep}{3pt}
\begin{tabular}{lccccccl}
\toprule
Readout & k{=}1 & k{=}2 & k{=}4 & k{=}8 & k{=}16 & $r_s$ & $p$ \\
\midrule
flatten       & 79.0 & 79.0 & 79.0 & 79.0 & 79.0 & $\sim\!0$ & --- \\
bilinear      & 78.1 & 78.9 & 78.9 & 78.1 & 78.1 & $+0.04$ & $0.63$ \\
bilinear+GELU & 78.9 & 79.7 & 79.7 & 79.7 & 79.7 & $-0.13$ & $0.14$ \\
svd-aug       & 77.3 & 78.1 & 78.1 & 77.3 & 78.1 & $+0.02$ & $0.82$ \\
quadratic     & 79.7 & 79.7 & 79.7 & 79.7 & 79.7 & $+0.07$ & $0.46$ \\
\bottomrule
\end{tabular}
\end{table}

All four $p$-values are above $0.14$; the Quadratic readout is
identical across all five $k$. The SVD-augmented readout, which
exposes singular values directly to the optimizer, does not produce
a rank-dependent curve.

\subsection{The mechanism is not just readout linearity}

Readouts with non-constant Jacobians still produce flat rank-$k$
curves. Proposition~\ref{prop:jac} is sufficient but not necessary
to cause them.

A plausible refinement: the trained readout's Jacobian at test
inputs has an effectively rank-1 active subspace in $\Z$ regardless
of whether the readout is in-principle nonlinear. In the absence of
an objective term that rewards rank, every readout family tested
admits a rank-1 shortcut and the optimizer takes it. We have not
yet measured $\mathrm{erank}(J(\Z))$ on these checkpoints to test
this directly.

Combined with the three-seed decoupling (effective ranks
$\{4,12,13\}$ at matched accuracy, Fig.~\ref{fig:seed-decoupling}),
the rank of $\Z$ in matrix-CODI is a free direction in the loss
landscape, and the four readouts above do not constrain it.
Candidates that might (we do not test them): an explicit rank reward
in the loss, tasks with verifiable multi-rank ground truth, or
step-level supervision in the spirit of SIM-CoT
\citep{shen2025simcot}.

\subsection{Negative control: rank-$k$ ablation on vanilla GPT-2 SFT}
\label{sec:pc-negctrl}

If the flat rank-$k$ curves were specific to the matrix-bottleneck
objective, running the same probe on a model with no bottleneck
should bend them. We run that test on a vanilla GPT-2 small
fine-tuned for ProsQA via standard supervised fine-tuning
(\textsc{pure-sft}; no latent tokens, no $\Z$, no distillation).
Three seeds $\{1337, 42, 7\}$ trained to $\sim\!79$pp ProsQA accuracy
matching the paper's vanilla baseline. We then \emph{construct} a
fake $\Z$ at inference by reshaping the first $256$ dimensions of
$h$ into a $16\!\times\!16$ matrix at the six token positions
immediately preceding the answer-prefix colon (the analog of
matrix-CODI's six latent positions), apply rank-$k$ truncation to
the fake $\Z$ via SVD, and propagate the modified residual through
the remaining transformer blocks. Decoding uses no KV cache, so the
intervention re-fires at every step.

\begin{table}[t]
\caption{Negative control: rank-$k$ ablation on vanilla GPT-2 SFT
(no matrix bottleneck, no $\Z$). Fake $\Z$ built from
$h[:\!256]$ reshaped to $16\!\times\!16$ at the six analog-latent
positions. Pooled-mean range across $k$ is $0.20$pp. Per-seed
Spearman $r_s$: $+0.32, -0.16, +0.71$ ($n\!=\!500$ test problems
each).}
\label{tab:pc-negctrl}
\centering
\small
\setlength{\tabcolsep}{4pt}
\begin{tabular}{lccccc}
\toprule
seed & $k=1$ & $k=2$ & $k=4$ & $k=8$ & $k=16$ \\
\midrule
$1337$ & $79.80$ & $80.20$ & $80.00$ & $80.20$ & $80.00$ \\
$42$   & $79.00$ & $78.80$ & $78.60$ & $78.60$ & $79.00$ \\
$7$    & $78.00$ & $78.00$ & $78.00$ & $78.00$ & $78.20$ \\
\midrule
pooled & $78.93$ & $79.00$ & $78.87$ & $78.93$ & $79.07$ \\
\bottomrule
\end{tabular}
\end{table}

The pooled-mean range across $k$ is $0.20$pp
(Table~\ref{tab:pc-negctrl}; Fig.~\ref{fig:pc-negctrl} overlays
per-seed and pooled curves on the matrix-CODI flatten reference).
As a sensitivity floor, replacing $h$ at the same six positions with
i.i.d.\ Gaussian noise matched in mean and standard deviation
produces seed accuracies $\{79.6, 79.2, 78.2\}$pp, statistically
indistinguishable from the unablated $\{80.0, 78.8, 78.0\}$pp. The
intervention paradigm is uninformative on this model.

\begin{figure}[t]
\centering
\includegraphics[width=\linewidth]{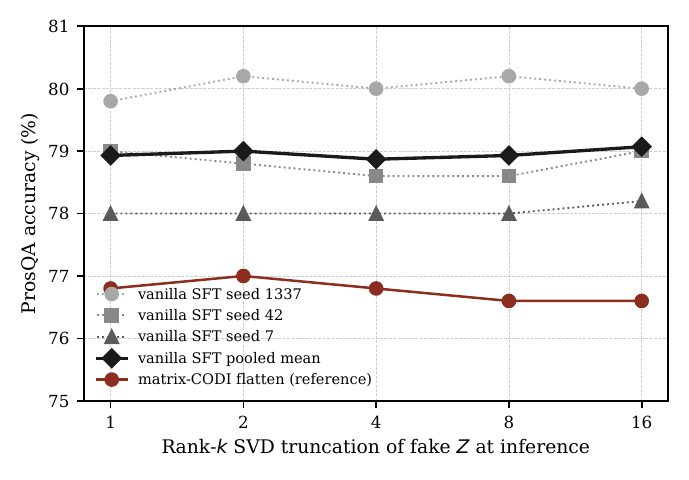}
\caption{Negative control: rank-$k$ ablation on vanilla GPT-2 SFT
(no matrix bottleneck, no $\Z$). Each gray dotted line is one seed
of the vanilla model; the heavy black line is the pooled mean
across the three seeds. The matrix-CODI flatten curve from
\S\ref{sec:flat} is overlaid in red as the reference. Both curves
are flat to within a fraction of a percentage point.}
\label{fig:pc-negctrl}
\end{figure}

A flat rank-$k$ curve is consistent with two states: a rank-blind
objective, or positions that do not carry the task's information.
Vanilla SFT is in the second by construction. The rank-$k$ ablation
alone cannot distinguish them.

In matrix-CODI, the bottleneck forces information through $\Z$ at
those positions during training, and the trained $\Z$ reaches
effective rank $12$--$13$ at $\gamma=0$. The seed-decoupling result
(Fig.~\ref{fig:seed-decoupling}) is a model-level property a
position-irrelevance reading would not predict.

\section{Related Work}
\label{sec:related}

\paragraph{Latent CoT and the Illusion of Superposition.}
\citet{rizvi2026illusion}'s fine-tuned COCONUT reaches $96.6\%$ on
ProsQA \emph{without} latent feedback, against $99.0\%$ with it.
Our vanilla SFT at GPT-2 small reaches $81.77\%$, roughly $15$pp
below that; we did not close the gap. The qualitative phenomenon
replicates at our operating point under the $\gamma=1$
(distillation-loss-active) configuration: matrix-CODI reaches
$82.03\%$ against pure SFT's $81.77\%$, a single-seed comparison. The
$\gamma=0$ configuration used in \S\ref{sec:flat}--\S\ref{sec:depth-scale}
removes the distillation loss and averages $80.99\pm2.0$pp over three
seeds ($78.91/81.25/82.81\%$); the $\gamma=1$ point falls inside that
spread. Latent feedback does not move accuracy beyond vanilla SFT under
either setting. Our
contribution relative to that work is a structural argument about
the training objective: the matrix-bottleneck produces
rank-indifferent gradients, and four nonlinear-in-$\Z$ readouts
fail to bend the rank-$k$ curve.

\paragraph{SIM-CoT.}
\citet{shen2025simcot} diagnose latent CoT instability as
insufficient step-level supervision and propose injecting per-step
targets. Our diagnosis is at a different layer (the matrix
bottleneck's objective produces rank-indifferent gradients,
adjudicated by four positive-control readouts). The two diagnoses
are consistent with both mechanisms operating simultaneously.

\paragraph{Reasoning by Superposition and CoT2.}
\citet{zhu2025superposition} prove that a two-layer transformer
with $D$ steps of continuous thought can solve directed graph
reachability, with each thought encoding a parallel BFS frontier.
\citet{gozeten2025cot2} show similar parallel-exploration behavior
under a GRPO-style training regime. Both are theoretical capacity
results with small empirical demonstrations. Capacity and
what-gets-learned are distinct; our result is about what CODI
distillation shapes, not whether transformers can in principle
encode superposition.

\paragraph{February 2026 rank measurements (direct adjacency).}
\citet{nazari2026rank} measure the effective rank of linear
attention hidden states and propose post-training rank pruning of
$K$ and $Q$ matrices. \citet{anonymous2026staterank} report
``state-rank stratification'' during pretraining: linear-attention
heads bifurcate into persistently low-rank and high-rank groups.
Both papers measure rank in the \emph{fast-weight memory} inside
an attention layer (a $d \times d$ accumulator), and both make
\emph{descriptive} claims about what trained networks end up with.
The object of study here is different (the explicit per-position
matrix latents $\Z$ on the matrix-CODI feedback path, not a
fast-weight memory inside attention), and the claim is a
mechanism claim about the training objective
(Proposition~\ref{prop:jac}) that we test by constructing four
nonlinear-in-$\Z$ positive controls.

\paragraph{Dynamics within latent CoT.}
\citet{anonymous2026dynamics} run multiple intervention protocols
on latent CoT hidden states (zero, mean, step-wise mean, Gaussian
noise) and an early-stop decoding that truncates latent
computation after step $k$. Their early-stop decoding is analogous
to our rank-$k$ ablation on a different axis: step depth vs.\
spectral truncation. Their step-wise causal structure is consistent
with our linear-probe decay $\Z[1]\to\Z[5]$ (early positions carry
information).

\paragraph{Rank-trajectory probing in depth-recurrent latent CoT.}
\citet{lu2025huginn} probe rank trajectories across recurrent
blocks in Huginn-3.5B, a depth-recurrent transformer that reuses
layers at inference. They find limited evidence of interpretable
latent CoT via rank-trajectory analysis. Our setting differs in
both model class (decoder-only GPT-2 with explicit per-position
$\Z$, not depth-recurrent shared-weight blocks) and observable
(rank-$k$ truncation of the trained $\Z$, not the rank trajectory
across recurrent depth), but the two papers reach a kindred
negative reading: rank-based probing of latent CoT does not, in
either setting, straightforwardly reveal a multi-path
superposition picture.

\paragraph{Rank decay and token uniformity in stacked attention.}
\citet{dong2021attention} showed that pure attention loses rank
doubly exponentially with depth. \citet{yan2022uniformity}
characterize the related token-uniformity phenomenon in
BERT-family encoders via the singular-value distribution of layer
outputs and propose a transformation that flattens the spectrum.
Both lines measure rank in the activations of a stack of
attention layers; ours measures the rank of an explicit matrix
latent on a feedback path. The two objects of study are distinct.

\paragraph{Implicit low-rank bias.}
Gradient descent on matrix-factorization losses has an implicit
bias toward low-rank solutions
\citep{gunasekar2017implicit,arora2019implicit,razin2020implicit};
\citet{kobayashi2024weightdecay} attribute this specifically to
weight decay. These concern bias through the parameter space, not
the readout. Our three-seed rank spread $\{4,12,13\}$ is consistent
with a weak attractor and inconsistent with strong collapse, which
would concentrate all seeds at the same low rank.

\paragraph{Alternative substrates and probe critique.}
\citet{wang2025latentvocab} argue that latent reasoning lives in
the vocabulary column space, not in hidden-state SVD directions;
if so, rank-$k$ truncation targets the wrong observable.
\citet{li2024optimal} note that zero/resample ablations (rank-$k$
truncation is one) overestimate component importance vs.\ optimal
ablation, which would only flatten our curves further. The
negative control in \S\ref{sec:pc-negctrl} is the empirical
counterpart.

\section{Discussion and Limitations}
\label{sec:discussion}

\paragraph{Single task, single architecture family.} All core
experiments run on ProsQA \citep{hao2024coconut} on GPT-2
$\{$small, medium, large$\}$. The GSM8K-Aug result in
Table~\ref{tab:four-flat} is at a $6\%$ operating point where the
model is barely learning the task; it is not strong evidence on
its own. The structural argument
(Proposition~\ref{prop:jac}) is stated in terms of the readout
$\phi$ and is therefore architecture-agnostic, but the empirical
evidence covers only this scale family. Cross-dataset replication
on GSM8K at a higher-accuracy operating point is pending.

\paragraph{Seed-dependent $\Z$ rank.} The three-seed decoupling
(Fig.~\ref{fig:seed-decoupling}) is a separate finding: the same
configuration, varying only the seed, produces models at effective
ranks $\{4, 12, 13\}$ with accuracies $\{81.25, 82.81, 78.91\}$.
Three seeds do not give statistical power to claim the rank
distribution is flat or uniform; the narrower claim is that seeds
at otherwise identical hyperparameters converge to materially
different effective ranks, inconsistent with a strong loss-side
preference for a specific rank. Implicit regularization from the
optimizer (Adam $+$ weight decay; \S\ref{sec:related}) may still
shape rank through channels outside $\mathcal{L}$. An $n\!=\!10$
replication is pending.

\paragraph{One seed per positive control.} The four
positive-control variants in \S\ref{sec:pc} were each trained once
(compute-bounded). Spearman $p$-values are computed on $128$ test
problems per checkpoint, which limits power for small effects. A
three-seed replication per variant ($\sim\!42$ H100-hours) and
re-running the four positive-control rank-$k$ evaluations on the
full $500$-problem ProsQA test set are both pending; the
$500$-problem eval raises power to detect $|r_s|\!\geq\!0.15$ from
$\sim\!40\%$ to $\sim\!80\%$ at $\alpha\!=\!0.05$.

\paragraph{Alternative explanation: the task is rank-1-solvable.}
ProsQA has a unique positive answer and a single distractor. If
all answer-predictive information lives in one singular direction
of $\Z$, every architecture would converge to a rank-1 functional
solution and rank-$k$ truncation would be flat. Our data are
consistent with that. What the strong reading does not explain
is that the trained $\Z$ reaches effective rank $12$--$13$ at
$\gamma=0$ instead of collapsing to 1, and three seeds spread to
$\{4,12,13\}$ rather than concentrating. The model builds rank it
does not functionally use, and the rank it builds is
seed-dependent. A reasoning task whose ground truth provably
requires $k>1$ independent quantities at the answer position would
disambiguate; we do not have one at this scale.

A result that would revise our reading: a readout that bends the
rank-$k$ curve on ProsQA (or a comparable structured task) under a
matrix-bottleneck objective. The four readouts in \S\ref{sec:pc}
were chosen to maximize the chance of seeing one. A different
training objective that explicitly rewards rank is a separate
direction we do not address.

\section{Conclusion}
\label{sec:conclusion}

The matrix-bottleneck training objective in CODI does not reward
rank: the readout Jacobian carries no rank information through the
chain rule, the flat rank-$k$ curves are insensitive to nonlinear
readouts that escape the linear-Jacobian shortcut, and three seeds
under matched hyperparameters land at effective ranks
$\{4, 12, 13\}$ with statistically indistinguishable accuracy. The
rank-$k$ probe alone could not distinguish rank-blindness from
position-irrelevance; the seed-level rank spread does.

\section{Reproducibility}
\label{sec:repro}

All training, evaluation, and probe code is released
at \url{https://github.com/saml212/matrix-codi-rank-blindness}.
The release includes:

\begin{itemize}
\setlength\itemsep{2pt}
\item \texttt{run\_matrix\_codi.py}: the matrix-CODI training
      and rank-$k$ evaluation script. The
      \texttt{MatrixBottleneck} class implements the
      $W_{\text{up}}\!\to\!\text{reshape}\!\to\!\text{thinker}\!\to\!\text{flatten}\!\to\!W_{\text{down}}$
      pipeline; all five readouts in \S\ref{sec:pc} are selectable
      via the \texttt{--readout} flag. The same script computes the
      accuracy-vs-$k$ curve and per-sample Spearman correlation
      between effective rank and correctness on a saved checkpoint.
\item \texttt{probe\_z.py}: the linear probe pipeline for
      \S\ref{sec:flat}. Produces the AUC numbers in
      Fig.~\ref{fig:probe}.
\item Raw rank-$k$ evaluation output files (JSON) for the four
      positive-control readouts, matching Table~\ref{tab:pc} and
      Fig.~\ref{fig:pc}.
\item A human-readable experiment log with per-run hyperparameters,
      rank trajectories, and wall-clock times.
\end{itemize}

Datasets: ProsQA \citep{hao2024coconut} and GSM8K-Aug
\citep{shen2025codi}, both from the original releases.

Backbone: pretrained GPT-2 small (124M), medium (355M), and large
(774M) from the standard public checkpoints.

Training hardware: a single NVIDIA H100 (80GB HBM3) per run. All
reported numerical results in the main body trace to a specific
checkpoint and evaluator run in the release.

\paragraph{Headline numbers by source.} Table~\ref{tab:four-flat}
rows R1--R3b come from \texttt{run\_matrix\_codi.py} at seed
$1337$; Fig.~\ref{fig:seed-decoupling} from the same script at
seeds $1337$, $42$, and $7$. Table~\ref{tab:pc} and
Fig.~\ref{fig:pc} are from \texttt{run\_matrix\_codi.py} run once
per readout (flatten, bilinear, bilinear+GELU, SVD-augmented,
quadratic), evaluated via the same script's rank-$k$ mode.
Fig.~\ref{fig:probe} is from \texttt{probe\_z.py} on the Round~3
$\gamma\!=\!0$ checkpoint. The preliminary depth result in
\S\ref{sec:depth-scale} and Fig.~\ref{fig:scale} are from the
vanilla SFT and vanilla CODI scripts in the same release.

\section*{Impact Statement}

This paper presents a negative empirical result on a specific
mechanistic-interpretability probe (the rank-$k$ ablation curve
applied to matrix-valued continuous chain-of-thought latents under
CODI-style distillation). The work advances the field of machine
learning by clarifying that one widely cited single-sample
structural observable does not, in this training regime, measure
what its name suggests. There are many potential societal
consequences of advancing machine-learning interpretability, none
of which we feel must be specifically highlighted here.

\bibliographystyle{icml2026}
\bibliography{refs}

\begin{thebibliography}{18}
\providecommand{\natexlab}[1]{#1}
\providecommand{\url}[1]{\texttt{#1}}
\expandafter\ifx\csname urlstyle\endcsname\relax
  \providecommand{\doi}[1]{doi: #1}\else
  \providecommand{\doi}{doi: \begingroup \urlstyle{rm}\Url}\fi

\bibitem[Arora et~al.(2019)Arora, Cohen, Hu, and Luo]{arora2019implicit}
Arora, S., Cohen, N., Hu, W., and Luo, Y.
\newblock Implicit regularization in deep matrix factorization.
\newblock In \emph{Advances in Neural Information Processing Systems
  (NeurIPS)}, 2019.
\newblock arXiv:1905.13655.

\bibitem[Deng et~al.(2025)Deng, Pang, Wei, Xu, Duan, Xu, Song, Shen, and
  Cheng]{wang2025latentvocab}
Deng, J., Pang, L., Wei, Z., Xu, S., Duan, Z., Xu, K., Song, Y., Shen, H., and
  Cheng, X.
\newblock {LLM} latent reasoning as chain of superposition.
\newblock \emph{arXiv preprint arXiv:2510.15522}, 2025.

\bibitem[Dong et~al.(2021)Dong, Cordonnier, and Loukas]{dong2021attention}
Dong, Y., Cordonnier, J.-B., and Loukas, A.
\newblock Attention is not all you need: Pure attention loses rank doubly
  exponentially with depth.
\newblock In \emph{International Conference on Machine Learning (ICML)}, 2021.
\newblock arXiv:2103.03404.

\bibitem[Gozeten et~al.(2025)Gozeten, Ildiz, Zhang, Harutyunyan, Rawat, and
  Oymak]{gozeten2025cot2}
Gozeten, A., Ildiz, M.~E., Zhang, Y., Harutyunyan, H., Rawat, A.~S., and Oymak,
  S.
\newblock Continuous chain of thought enables parallel exploration and
  reasoning.
\newblock \emph{arXiv preprint arXiv:2505.23648}, 2025.

\bibitem[Gunasekar et~al.(2017)Gunasekar, Woodworth, Bhojanapalli, Neyshabur,
  and Srebro]{gunasekar2017implicit}
Gunasekar, S., Woodworth, B., Bhojanapalli, S., Neyshabur, B., and Srebro, N.
\newblock Implicit regularization in matrix factorization.
\newblock In \emph{Advances in Neural Information Processing Systems
  (NeurIPS)}, 2017.
\newblock arXiv:1705.09280.

\bibitem[Hao et~al.(2024)Hao, Sukhbaatar, Su, Li, Hu, Weston, and
  Tian]{hao2024coconut}
Hao, S., Sukhbaatar, S., Su, D., Li, X., Hu, Z., Weston, J., and Tian, Y.
\newblock Training large language models to reason in a continuous latent
  space.
\newblock \emph{arXiv preprint arXiv:2412.06769}, 2024.

\bibitem[Kobayashi et~al.(2024)Kobayashi, Akram, and von
  Oswald]{kobayashi2024weightdecay}
Kobayashi, S., Akram, Y., and von Oswald, J.
\newblock Weight decay induces low-rank attention layers.
\newblock In \emph{Advances in Neural Information Processing Systems
  (NeurIPS)}, 2024.
\newblock arXiv:2410.23819.

\bibitem[Li \& Janson(2024)Li and Janson]{li2024optimal}
Li, M. and Janson, L.
\newblock Optimal ablation for interpretability.
\newblock In \emph{Advances in Neural Information Processing Systems
  (NeurIPS)}, 2024.
\newblock arXiv:2409.09951.

\bibitem[Li et~al.(2026)Li, Bai, Chen, Li, Yang, Lin, and
  Zhang]{anonymous2026dynamics}
Li, Z., Bai, X., Chen, K., Li, Y., Yang, J., Lin, C., and Zhang, M.
\newblock Dynamics within latent chain-of-thought: An empirical study of causal
  structure.
\newblock \emph{arXiv preprint arXiv:2602.08783}, 2026.

\bibitem[Lu et~al.(2025)Lu, Yang, Lee, Li, and Liu]{lu2025huginn}
Lu, W., Yang, Y., Lee, K., Li, Y., and Liu, E.
\newblock Latent chain-of-thought? decoding the depth-recurrent transformer.
\newblock \emph{arXiv preprint arXiv:2507.02199}, 2025.

\bibitem[Nazari \& Rusch(2026)Nazari and Rusch]{nazari2026rank}
Nazari, P. and Rusch, T.~K.
\newblock The key to state reduction in linear attention: A rank-based
  perspective.
\newblock \emph{arXiv preprint arXiv:2602.04852}, 2026.

\bibitem[Razin \& Cohen(2020)Razin and Cohen]{razin2020implicit}
Razin, N. and Cohen, N.
\newblock Implicit regularization in deep learning may not be explainable by
  norms.
\newblock In \emph{Advances in Neural Information Processing Systems
  (NeurIPS)}, 2020.
\newblock arXiv:2005.06398.

\bibitem[Rizvi-Martel et~al.(2026)Rizvi-Martel, Rabusseau, and
  Mosbach]{rizvi2026illusion}
Rizvi-Martel, M., Rabusseau, G., and Mosbach, M.
\newblock The illusion of superposition? a principled analysis of latent
  thinking in language models.
\newblock \emph{arXiv preprint arXiv:2604.06374}, 2026.

\bibitem[Shen et~al.(2025)Shen, Yan, Zhang, Hu, Du, and He]{shen2025codi}
Shen, Z., Yan, H., Zhang, L., Hu, Z., Du, Y., and He, Y.
\newblock {CODI}: Compressing chain-of-thought into continuous space via
  self-distillation.
\newblock In \emph{Proceedings of the Conference on Empirical Methods in
  Natural Language Processing (EMNLP)}, 2025.
\newblock arXiv:2502.21074.

\bibitem[Shen et~al.(2026)]{shen2025simcot}
Shen, Z. et~al.
\newblock {SIM-CoT}: Step-level implicit supervision for continuous chain of
  thought.
\newblock In \emph{International Conference on Learning Representations
  (ICLR)}, 2026.
\newblock arXiv:2509.20317.

\bibitem[Sun et~al.(2026)Sun, Zhang, Zhou, Ma, Qin, Su, Liu, Ma, Xu, Gao, Hao,
  and He]{anonymous2026staterank}
Sun, A., Zhang, H., Zhou, H., Ma, Y., Qin, Y., Su, T., Liu, Y., Ma, Z., Xu, J.,
  Gao, J., Hao, J., and He, R.
\newblock State rank dynamics in linear attention {LLMs}.
\newblock \emph{arXiv preprint arXiv:2602.02195}, 2026.

\bibitem[Yan et~al.(2022)Yan, Gui, Li, and He]{yan2022uniformity}
Yan, H., Gui, L., Li, W., and He, Y.
\newblock Addressing token uniformity in transformers via singular value
  transformation.
\newblock \emph{arXiv preprint arXiv:2208.11790}, 2022.

\bibitem[Zhu et~al.(2025)Zhu, Hao, Hu, Jiao, Russell, and
  Tian]{zhu2025superposition}
Zhu, H., Hao, S., Hu, Z., Jiao, J., Russell, S., and Tian, Y.
\newblock Reasoning by superposition: A theoretical perspective on chain of
  continuous thought.
\newblock In \emph{Advances in Neural Information Processing Systems
  (NeurIPS)}, 2025.
\newblock arXiv:2505.12514.

\end{thebibliography}

\end{document}